\documentclass{article}

\usepackage[preprint, nonatbib]{neurips_2023}
\usepackage[sort&compress,square,numbers]{natbib}
\usepackage[utf8]{inputenc} 
\usepackage[T1]{fontenc}    
\usepackage[colorlinks,linktoc=all]{hyperref}       
\usepackage[all]{hypcap}
\hypersetup{citecolor=MidnightBlue}
\hypersetup{urlcolor=MidnightBlue}
\hypersetup{linkcolor=black}

\usepackage{url}            
\usepackage{booktabs}       
\usepackage{amsfonts}       
\usepackage{nicefrac}       
\usepackage{microtype}      
\usepackage{caption}

\usepackage{amsfonts}
\usepackage{color}
\usepackage[usenames,dvipsnames]{xcolor}
\usepackage{multirow}
\usepackage{graphicx}
\usepackage{amssymb} 
\usepackage{enumitem}
\usepackage{fullpage}
\usepackage[titlenumbered,ruled]{algorithm2e}
\usepackage[many]{tcolorbox}
\usepackage{empheq}
\usepackage{amsthm}
\usepackage{amsbsy}
\usepackage{mathtools}
\usepackage{thm-restate}
\usepackage{mdframed}
\usepackage{thmtools}
\usepackage{subcaption}

\newcommand{\blue}[1]{\textcolor[rgb]{0,0.33,0.56}{#1}}
\newcommand{\green}[1]{\textcolor[rgb]{0,0.65,0.33}{#1}}
\definecolor{shade}{rgb}{0.9,0.9,0.9}
\usepackage{bm}
\renewcommand{\vec}[1]{\bm{#1}}
\newcommand{\mat}[1]{\bm{#1}}
\newcommand{\revised}[1]{\textcolor{black}{ #1}}
\newcommand{\vmu}{\vec{\mu}}
\newcommand{\vSigma}{\vec{\Sigma}}
\newcommand{\vtheta}{\vec{\theta}}
\newcommand{\vw}{\vec{w}}
\newcommand{\vx}{\vec{x}}

\renewcommand*{\Re}{\mathbb{R}}
\graphicspath{{figs/}}
\graphicspath{{../}}

\newcommand{\bw}{\vec{w}}

\newcommand{\red}[1]{\textcolor[rgb]{0.635,0.0780,0.1840}{#1}}
\newcommand{\dotprod}[1]{\left< #1\right>}

\newcommand{\mI}{{\bf I}}

\definecolor{shadecolor}{gray}{0.90}
\declaretheoremstyle[
headfont=\normalfont\bfseries,
notefont=\mdseries, notebraces={(}{)},
bodyfont=\normalfont,
postheadspace=0.5em,
spaceabove=6pt,
mdframed={
  skipabove=8pt,
  skipbelow=8pt,
  hidealllines=true,
  backgroundcolor={shadecolor},
  innerleftmargin=4pt,
  innerrightmargin=4pt}
]{shaded}

\usepackage[colorinlistoftodos,bordercolor=orange,backgroundcolor=orange!20,linecolor=orange]{todonotes}

\newcommand{\rmq}{q}
\newcommand{\rmp}{p}
\newcommand{\g}{\, | \,}

\newcommand{\grad}[1]{\nabla_{#1} \,}

\newcommand{\kl}[1]{\textrm{KL}\left(#1\right)}

\usepackage[nameinlink,capitalise]{cleveref}
\creflabelformat{equation}{#1#2#3}
\crefname{equation}{eq.}{eqs.}
\Crefname{equation}{Eq.}{Eqs.}
\usepackage{framed}

\title{Variational Inference with Gaussian Score Matching}

\author{%
  Chirag Modi\\
  Center for Computational Astrophysics, \\
  Center for Computational Mathematics,  \\
  Flatiron Institute, New York \\
  \texttt{cmodi@flatironinstitute.org} \\
  \And
  Charles Margossian\\
  Center for Computational Mathematics,  \\
  Flatiron Institute, New York \\
  \texttt{cmargossian@flatironinstitute.org} \\
  \And
  Yuling Yao\\
  Center for Computational Mathematics,  \\
  Flatiron Institute, New York \\
  \texttt{yyao@flatironinstitute.org} \\
  \And
  Robert Gower\\
  Center for Computational Mathematics,  \\
  Flatiron Institute, New York \\
  \texttt{rgower@flatironinstitute.org} \\
  \And
  David Blei\\
  Department of Computer Science, Statistics,\\
  Columbia University, New York \\
  \texttt{david.blei@columbia.edu} \\
  \And
  Lawrence Saul\\
Center for Computational Mathematics,  \\
  Flatiron Institute, New York\\
  \texttt{lsaul@flatironinstitute.org} \\}

\begin{document}

\maketitle

\begin{abstract}
  Variational inference (VI) is a method to approximate the
  computationally intractable posterior distributions that arise in
  Bayesian statistics. Typically, VI fits a simple parametric distribution to be close to the target posterior, minimizing an appropriate objective such as the evidence lower bound (ELBO).
  In this work, we present a new approach to VI. Our method is based on the principle of score matching, that if two distributions are equal then their score functions (i.e., gradients of the log density) are equal at every point on their support. With this
  principle, we develop score matching VI, an iterative algorithm that seeks to match the scores between the variational approximation and the exact posterior. At each iteration, score matching VI solves an inner optimization, one that minimally adjusts the current variational estimate to match the scores at a newly sampled value of the latent variables. We show that when the variational family is a  Gaussian, this inner optimization enjoys a closed form solution, which we call Gaussian score matching VI (GSM-VI). GSM-VI is also a ``black box'' variational algorithm in that it only requires a differentiable joint distribution, and as such it can be applied to a wide class of models. We compare GSM-VI to black box variational inference (BBVI), which has similar requirements but instead optimizes the ELBO. 
  We first study how GSM-VI behaves as a function of the problem
  dimensionality, the condition number of the target covariance matrix (when the target is Gaussian), and the degree of mismatch between the approximating and exact posterior distribution. We then study GSM-VI on a collection of real-world Bayesian inference problems from the posteriorDB database of datasets and models. In all of our studies we find that GSM-VI is faster than BBVI, but without sacrificing accuracy. It requires 10-100x fewer gradient evaluations to obtain a comparable quality of approximation\footnote{We provide a Python implementation of GSM-VI algorithm at \href{https://github.com/modichirag/GSM-VI}{https://github.com/modichirag/GSM-VI}.}.
\end{abstract}

\section{Introduction}

This paper is about variational inference for approximate Bayesian computation. Consider a statistical model $\rmp(\vtheta, \vx)$ of parameters $\vtheta \in \Re^d$ and observations $\vx$. Bayesian inference aims to infer the posterior distribution $p(\vtheta \g \vx)$, which is often intractable to compute.  Variational inference is an optimization-based approach to approximate the posterior~\cite{blei2017variational,Jordan1999}.

The idea behind VI is to approximate the posterior with a member of a \textit{variational family} of distributions $q_{\vw}(\vtheta)$, parameterized by \textit{variational parameters} $\vw$~\cite{blei2017variational,Jordan1999}. Specifically, VI methods establish a measure of closeness between $q_{\vw}(\vtheta)$ and the posterior, and then minimize it  with an optimization algorithm. Researchers have explored many aspects of VI, including different objectives~\citep{Knoblauch:2019,Ranganath:2016a,Yang2019,Dhaka2021,Minka:2001,Dieng:2017,Naesseth:2020} and optimization strategies~\citep{Agrawal2020,Ranganath2014,Hoffman:2013}.

In its modern form, VI typically minimizes $\kl{q_{\vw}(\theta) || \rmp(\theta \g x)}$ with stochastic optimization, and further satisfies the so-called ``black-box'' criteria~\citep{Agrawal2020,Ranganath2014,welandawe2022robust}.
The resulting black-box VI (BBVI) only requires the practitioner to specify the log joint $\log \rmp(\vtheta, \vx)$ and (often) its gradient $\nabla_{\vtheta} \log \rmp(\theta, \vx)$, which for many models can be obtained by automatic differentiation. For these reasons, BBVI has been widely implemented, and it is 
available in many 
probabilistic programming systems ~\citep{Bingham:2018,Kucukelbir2016,Salvatier:2016}.

In this paper, we propose a new approach to VI. We begin with the principle of \textit{score matching}~\cite{Hyvarinen:2005}, that when two densities are equal then their gradients are equal as well, and we use this principle to derive a new way to fit a variational distribution to be close to the exact posterior. The result is \textit{score-matching VI}. Rather than explicitly minimize a divergence, score-matching VI iteratively projects the variational distribution onto the exact score matching constraint.  This strategy enables a new black-box VI algorithm.

Score-matching VI relies on the same ingredients as reparameterization BBVI~\citep{Kucukelbir2016}---a differentiable variational family and a differentiable log joint---and so it can be as easily incorporated into probabilistic programming systems as well. Further, when the variational family is a Gaussian, score-matching VI is particularly efficient: each iteration is computable in closed form. We call the resulting algorithm Gaussian score matching VI (GSM-VI).

Unlike BBVI, GSM-VI does not rely on stochastic gradient descent (SGD) for its core optimization. Though SGD has the appeal of simplicity, it is also known to require the careful tuning of learning rates. GSM-VI was inspired by a different tradition of constraint-based algorithms for online learning~\cite{Crammer06,Dredze08,ALI-G,TASPS,SPS}. These algorithms have been extensively developed and analyzed for problems in classification, and under the right conditions, they have been observed to outperform SGD. This paper shows how to extend this constraint-based framework---and the powerful machinery behind it---from the problem of classification to the workhorse of Gaussian VI. The key insight is that score-matching (unlike ELBO maximization) lends itself naturally to a constraint-based formulation.

We empirically compared GSM-VI to reparameterization BBVI on several classes of models, and with both synthetic and real-world data. In general, we found that GSM requires 10-100x fewer gradient evaluations to converge to an equally good approximation. When the exact posterior is Gaussian, we found GSM-VI scales significantly better with respect to dimensionality and is insensitive to the condition number of the target covariance. When the exact posterior is non-Gaussian, we found GSM-VI enjoys faster convergence without sacrificing the quality of the final approximation.



This paper makes the following contributions: \begin{itemize}[topsep=0mm]
\item We introduce \textit{score matching variational inference}, a new black-box approach to fitting $q_{\vw}(\vtheta)$ to be close to $\rmp(\vtheta \g \vx)$. Score matching VI requires no tunable optimization hypermarameters, to which BBVI can be sensitive.

\item When the variational family is Gaussian, we develop \textit{Gaussian score matching variational inference} (GSM-VI). It establishes efficient closed-form iterates for score matching VI.

\item We empirically compare GSM-VI to reparameterization BBVI. Across many models and datasets, we found that GSM-VI enjoys faster convergence to an equally good approximation.

\end{itemize}

We develop score matching VI in \Cref{sec:gsm} and study its performance in Section \ref{sec:exp}.


\textbf{Related work.} Our work introduces a new method for black-box variational inference that relies only on having access to the gradients of the variational distribution and the log joint.  GSM-VI has similar goals to automatic-differentiation variational inference (ADVI) \citep{Kucukelbir2016} and Pathfinder~\citep{Zhang2022}, which also fit multivariate Gaussian variational families, but do so by maximizing the ELBO using stochastic optimization. Similar to GSM-VI, the algorithm of ref. \citep{Seljak2019} also seeks to match the scores of the variational and the target posterior, but it does so by minimizing the L2 loss between them. 

A novel aspect of this work is how GSM-VI fits the variational
parameters. Rather than minimize a loss function, it aims to solve
a set of nonlinear equations.  Similar ideas have been pursued in the context of fitting a model to data using  empirical risk minimization (ERM). For example, passive agressive (PA)
methods~\cite{Crammer06} and the stochastic polyak stepsize (SPS) are
also derived via projections onto sampled nonlinear
equations~\cite{ALI-G,TASPS,SPS}. A probabilistic extension of PA methods is known as confidence-weighted (CW) learning~\cite{Dredze08}. In this framework, the learner maintains a multivariate Gaussian distribution over the weight vector of a linear classifier. Like CW learning, the second step of GSM-VI also minimizes a KL divergence between multivariate Gaussians. But it involves a different projection, one of score-matching versus linear classification.

y\section{Score Matching Variational Inference}
\label{sec:gsm}



Suppose for the moment that the variational family $q_{\bw}(\vtheta)$ is rich enough to perfectly capture the posterior $p(\vtheta \g \vx)$. In other words, there exists a $\vw^*$ such that \begin{align}
  \label{eq:interpolation}
  \log \, q_{\bw^*}(\vtheta)
  &= \log \, p(\vtheta | \vx), \qquad \forall \vtheta \, \in \Theta.
\end{align}
If we could solve \Cref{eq:interpolation} for $\bw^*$, the resulting variational distribution would be a perfect fit. The challenge is that the posterior on the right side is intractable to compute.

To help, we appeal to score matching~\citep{Hyvarinen:2005}.  Define the score of a distribution to be the gradient of its log with respect to the variable\footnote{We make this clear because, in some literature, the score is the gradient with respect to the parameter.}, e.g., $\nabla_{\vtheta} \log q_{\vw}(\vtheta)$. The principle of score matching is that if two distributions are equal at each point in their support then their score functions are also equal.

To use score matching for VI, we first write the log posterior as the log joint minus the normalizing constant, i.e., the marginal distribution of $\vx$,
\begin{align}
  \log \, p(\vtheta \g \vx) &= \log \, p(\vtheta, \vx) - \log \, p(\vx).
\end{align}
With this expression, the principle of score matching leads to the following Lemma.
\begin{restatable}{lemma}{scorematch} \label{lem:scorematch}
\revised{Let $\Theta \subset \Re^d$ be a set that contains the support of $p(\vec{\vtheta, \vx})$ and  $q_{\vec{w}^*}(\vtheta)$ for some $\vw^*.$  That is $p(\vec{\vtheta, \vx}) = q_{\vec{w}^*}(\vtheta) =0$ for every $\theta \not\in \Theta.$ Furthermore, suppose that $\Theta$ is path-connected. }
  The parameter $\vw^*$ satisfies
  \begin{align}
    \label{eq:scorematch}
    \nabla_{\vtheta} \log \, q_{\vec{w}^*}(\vtheta)
    &=
      \nabla_{\vtheta} \log \, p(\vec{\vtheta, \vx})
      ,\qquad
      \forall \vtheta \in \Theta,
  \end{align}
if and only if $\bw^*$ also satisfies \Cref{eq:interpolation}.
\end{restatable}
What is notable about \Cref{eq:scorematch} is that the right side is the gradient of the log joint. Unlike the posterior, the gradient of the log joint is tractable to compute for a large class of probabilistic models. (The proof is in the appendix.)

This lemma motivates a new algorithm, \textit{score matching VI}. 
The idea is to iteratively refine the variational parameters $\vw$ to try to solve the system of equations in \Cref{eq:scorematch} as well as possible. At each iteration $t$, it first samples a new $\vtheta_t$ from the current variational approximation and then minimally adjusts $\vw$ to satisfy \Cref{eq:scorematch} for that value of $\vtheta_t$.

\newpage
\begin{framed}
\vspace{-3ex}
  \textbf{Score matching variational inference} 

  At iteration $t$:
\begin{enumerate}[noitemsep, leftmargin=*,topsep=1mm]
  \vspace{-2.5mm}
  \item Sample $\vtheta_t \sim \rmq_{\bw_{t}}(\vtheta)$.  \item Update the variational parameters:
    \begin{align}
      \label{eq:KLproj}
      \begin{split}
        \bw_{t+1}
        = \arg & \min_{\bw} \kl{\rmq_{\bw_{t}}(\vtheta) \, || \, \rmq_{\bw}(\vtheta)}
        \\& 
        \textrm{such that} \quad
        \grad{\vtheta}
        \log \rmq_{\bw}(\vtheta_t) = \grad{\vtheta} \log \rmp(\vtheta_t, \vx).
      \end{split}
    \end{align}
  \end{enumerate}
\vspace{-2ex}
\end{framed}

\vspace{-1ex}
This algorithm for score matching VI was inspired by earlier online algorithms for learning a classifier from a stream of labeled examples. One particularly elegant algorithm in this setting is known as passive-aggressive (PA) learning~\citep{Crammer06}, in which a model is incrementally updated by the minimal amount to classify each example correctly by a large margin. This approach was subsequently extended to a probabilistic setting, known as confidence-weighted (CW) learning~\citep{Dredze08} in which one minimally updates a {\it distribution} over classifiers. Our algorithm is similar in that it minimally updates an approximating distribution for VI, but it is different in that enforces constraints for score matching instead of large margin classification.

At a high level, what makes this approach to VI likely to succeed or fail? Certainly it is necessary that there are more variational parameters than elements of the latent variable $\vtheta$; when this is not the case, it may be impossible to satisfy a {\it single} score matching constraint in \Cref{eq:KLproj}. That said, setting the number of variational parameters to be at least as large as the latent variable is standard, as in, e.g., a factorized (or mean-field) variational family. It is also apparent that the algorithm may never converge if the target posterior is not contained in the variational family, or that the the variational approximation collapses to a point mass, which stalls the updates altogether. While we cannot dismiss these possibilities out of hand, we did not see either issue in any of the empirical studies of \Cref{sec:exp}.




For more intuition, \Cref{fig:density} illustrates the effect of the update in \Cref{eq:KLproj} when both the target and approximating distribution are 1d Gaussian.
The target posterior $\rmp(\vtheta \g \vx)$ is shaded blue. The plot shows the initial variational distribution $\rmq_{\bw_0}$ (light grey curve) and its update to $\rmq_{\bw_1}$ (medium grey) so that the gradient of the updated distribution matches the gradient of the target at the sampled $\vtheta_0$ (dotted red tangent line). It also shows the update from $\bw_1$ to $\bw_2$, now matching the gradient at$\vtheta_1$. With these two updates, $\rmq_{\bw_2}$ (dark grey) is very close to the target $\rmp(\vtheta \g \vx)$. With this picture in mind, we now develop the details of this algorithm for a widely applicable setting.

\textbf{Gaussian Score Matching VI.} Suppose the variational distribution belongs to a multivariate Gaussian family $q_{\vec{w}}(\vtheta) := \mathcal{N}(\vmu, \vSigma)$, which is a common setting especially in systems for automated approximate inference~\citep{Kucukelbir2016,Agrawal2020}. One of our main contributions is to show that in this case \Cref{eq:KLproj} has a closed form solution. The solution $\bw_{t+1} = (\vec{\mu}_{t+1} , \mat{\Sigma}_{t+1})$
has the following form:



\begin{figure}[t]
\centering
\begin{subfigure}{.51\textwidth}
  \centering
  \includegraphics[width=1.0\textwidth,  bb=-0 -0 962 911]{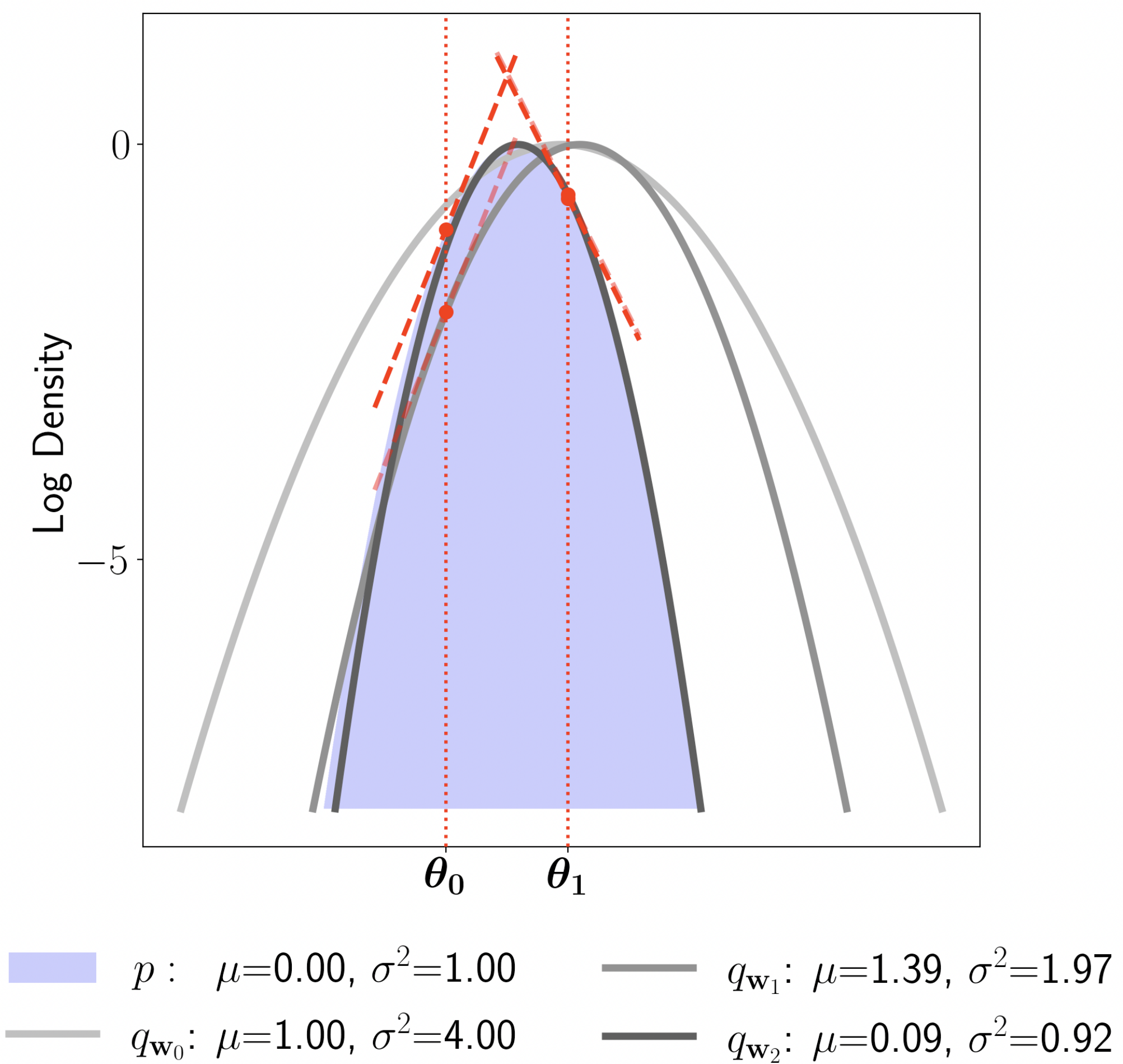}
  \caption{
  Two iterations of GSM-VI. The log density of the target posterior $p$ is shaded blue; the initial distribution $q_{\bw_0}$ is light grey; the first update $q_{\bw_1}$ is medium grey;
and the second update $q_{\bw_2}$ is dark grey.}
  \label{fig:density}
\end{subfigure} \hspace{0.03\textwidth}
\begin{subfigure}{.44\textwidth}
    \centering
    \includegraphics[width =1.0\textwidth, height = 1.00\textwidth, bb=-0 -0 888 864]{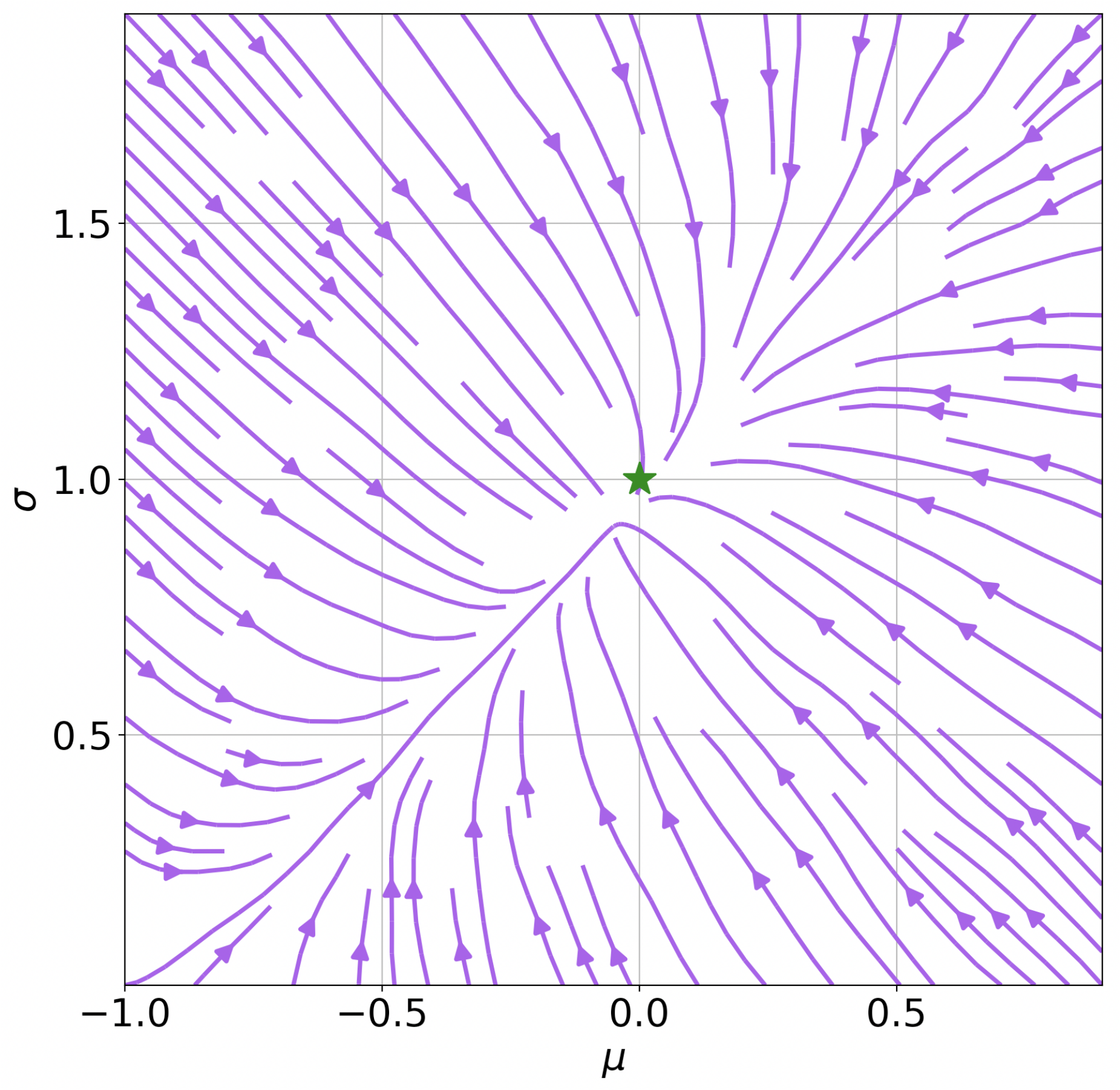}
   \caption{  The vector field of \Cref{eq:KLproj}, averaged over 5 independent samples at each point, where $p(\vtheta \g \vx) = \mathcal{N}(0,1)$. The solution $(\mu,\sigma) =(0,1)$ is the green star.}
    \label{fig:vec}
\end{subfigure}
\end{figure}

\newtcolorbox{empheqboxed}{colback=gray!20, 
 colframe=white,
 width=\textwidth,
 sharpish corners,
 top=1mm, 
 bottom=2mm
}
\begin{empheqboxed}
\vspace{-2.5ex}
\begin{align}
 \label{eq:muupdate-intro}
    \vec{\mu}_{t+1} &\; = \; \vec{\mu}_t +\mat{A}_t\, \left(\grad{\vtheta} \log \rmp(\vtheta_t, \vx) - \grad{\vtheta} \log \rmq_{\bw_{t}}(\vtheta_t) \right) \\
    \label{eq:sigamupdate-intro}
    \mat{\Sigma}_{t+1} &\; =\; \mat{\Sigma}_t + (\vec{\mu}_t-\vtheta_t)(\vec{\mu}_t-\vtheta_t)^\top - (\vec{\mu}_{t+1}-\vtheta_t)(\vec{\mu}_{t+1}-\vtheta_t)^\top   
    \end{align}
\vspace{-4.5ex}
\end{empheqboxed}
where $\mat{A}_t \in \Re^{d\times d}$ is a matrix defined in the theorem below. These exact updates only require the score of the log joint $\nabla_{\vtheta} \log p(\vtheta, \vx)$ and the score of the variational distribution $\nabla_{\vtheta} \log q_{\vw}(\vtheta)$.

\Cref{eq:muupdate-intro,eq:sigamupdate-intro} also provide intuition. Consider the approximation at the $t$th iteration $q_{\vw_t}$ and the current sample $\vtheta_t$. First suppose the scores already match at this sample, that is $\grad{\vtheta} \log \rmp(\vtheta_t, \vx) = \grad{\vtheta} \log \rmq_{\bw_{t}}(\vtheta_t)$. Then the mean does not change $ \vec{\mu}_{t+1} = \vec{\mu}_t$ and, similarly, the two rank-one terms in the covariance update in \Cref{eq:sigamupdate-intro} cancel out so $\mat{\Sigma}_{t+1} =\mat{\Sigma}_{t}$. This shows that when $q_{\bw_t}(\vtheta) =p(\vtheta, \vx)$ for all $\vtheta$, the method stops. On the other hand, if the scores do not match, then the mean is updated proportionally to the difference between the scores, and the covariance is updated by a  rank-two correction. For a one dimensional target $p(\vtheta , \vx) = \mathcal{N}(0,1)$, Figure~\ref{fig:vec} illustrates the vector field of updates  The vector field points to the solution (green star) and, once there, the method stops. 


We now formalize this result and give the exact expression for $\mat{A}_t$.


\begin{restatable}{theorem}{GSMtheo} \label{theo:GSM}
\textbf{(GSM-VI updates)} Let $p(\vtheta,\vec{x})$ be given for some $\vtheta\in\Re^d$, and let $q_{\bw^t}(\vtheta)$ and $q_{\bw}(\vtheta)$ be  multivariate normal distributions with means $\vec{\mu}_t$ and $\vec{\mu}$ and covariance matrices $\mat{\Sigma}_t$ and $\mat{\Sigma}$, respectively. As shorthand, let $\vec{g}_t := \nabla_{\vtheta} \log p(\vtheta_t,\vec{x})$ and let
\begin{equation}
\vec{\mu}_{t+1}, \mat{\Sigma}_{t+1} \; = \; \arg\!\min_{\!\!\!\!\!\!\!\vec{\mu},\mat{\Sigma}\succeq 0 
} \Big[ {\rm KL}(q_t,q)\Big] \quad\mbox{such that}\quad\nabla_{\vtheta} \log q(\vtheta_t) = \nabla_{\vtheta} \log p(\vtheta_t,\vec{x}).
\label{eq:GSMopt}
\end{equation}
The solution to eq.~(\ref{eq:GSMopt})
is given by \Cref{eq:muupdate-intro,eq:sigamupdate-intro} where 
\begin{eqnarray}\label{eq:At}
 \mat{A}_t & := & \frac{1}{1+\rho}\left[ \mI - \frac{(\vec{\mu}_t-\vtheta_t) \vec{g}_t^\top}{1+\rho+ (\vec{\mu}_t-\vtheta_t)^\top \vec{g}_t}\right]\mat{\Sigma}_t\,,
\end{eqnarray}
and  $\rho$ is the positive root of the quadratic equation
\begin{equation}
\rho(1\!+\!\rho)\ =\ \vec{g}_t^\top\mat{\Sigma}_t\vec{g}_t + \big[(\vec{\mu}_t\!-\!\vtheta_t)^\top\!\vec{g}_t\big]^2.
\end{equation}
\end{restatable}



With the definition of $\mat{A}_t$ in Eq. \ref{eq:At} we can see that the computational complexity of updating $\vec{\mu}$ and $\mat{\Sigma}$ via~\Cref{eq:muupdate-intro,eq:sigamupdate-intro} is $\mathcal{O}(d^2)$, where $\vtheta \in \Re^d$ and we assume the cost of computing the gradients is $\mathcal{O}(d).$ Note this is the best possible iteration complexity we can hope for, since we store and maintain the full covariance matrix of $d^2$ elements. (The proof is in the appendix.)

\Cref{alg:gsm} presents the full GSM-VI algorithm. Here we also use mini-batching, where we average over $B \in \mathbb{N}$ independently sampled updates of \Cref{eq:muupdate-intro,eq:sigamupdate-intro} before updating the mean and covariance.

\begin{algorithm}[t]
\label{alg:gsm}
\SetKwInOut{KwIn}{Input}
\SetKwInOut{KwOut}{Output}
\KwIn{Initial mean estimate $\vec{\mu}_0$, 
initial covariance estimate $\vec{\Sigma}_0$, 
target distribution $p(\vtheta| \vx)$, number of iterations $N\in \mathbb{N}$, batch size $B\in \mathbb{N}.$}
\SetKwComment{Comment}{$\triangleright$\ }{}

\KwOut{Multivariate normal variational distribution $q_{\bw}(\vtheta):=\mathcal{N}(\vec{\mu}, \vec{\Sigma}$)}

\For{$i = 0, \ldots,  N-1$ \quad \quad \Comment{iteration loop}  }{
    \For{$j =0 , \ldots,  B-1$ \quad \quad \Comment{batch loop}  }  { 
        Sample $\vtheta^{(j)}\sim \mathcal{N}(\vec{\mu_{i}}, \vec{\Sigma_{i}})$ 
    
        $\vec{g} \leftarrow \nabla_{\vtheta} \log p(\vtheta^{(j)} | \vx)$ 
    
        $\vec{\varepsilon}\ \leftarrow \ \mat{\Sigma}_i\vec{g} - \vec{\mu}_i + \vtheta$ 
    
        Solve $\rho(1\!+\!\rho)\ =\ \vec{g}^\top\mat{\Sigma}_i\vec{g} + \big[(\vec{\mu}_i\!-\!\vtheta)^\top\!\vec{g}\big]^2$ for $\rho>0$ 

        $ \vec{\delta\mu}^{(j)} \leftarrow \frac{1}{1+\rho}\left[ \mI - \frac{(\vec{\mu}_i-\vtheta) \vec{g}^\top}{1+\rho+ (\vec{\mu}_i-\vtheta)^\top \vec{g}}\right] \vec{\varepsilon}$ 

        $ \vec{\mu_i}^{(j)} \leftarrow \vec{\mu_i} + \vec{\delta\mu}^{(j)} $
    
        $ \vec{\delta\Sigma}^{(j)}  \leftarrow   (\vec{\mu}_i-\vtheta)(\vec{\mu}_i-\vtheta)^\top - (\vec{\mu_i}^{(j)}-\vtheta)(\vec{\mu_i}^{(j)}-\vtheta)^\top $  
    }
    
    Update $\vec{\mu_{i+1}} \leftarrow \vec{\mu_i} + \sum_j\vec{\delta\mu^{(j)}} / B$ 
    
    Update $\vec{\Sigma_{i+1}} \leftarrow \vec{\Sigma_{i}} + \sum_j\vec{\delta\Sigma^{(j)}} / B$ 
}
$q_{\bw}(\vtheta) \leftarrow \mathcal{N}( \vec{\mu}_{N}, \vec{\Sigma}_{N})$ 
\caption{Gaussian Score Matching VI}
\label{alg:gsm}
\end{algorithm}


\section{Empirical Studies}
\label{sec:exp}

\begin{algorithm}[t]
\label{alg:bbvi}
\SetKwInOut{KwIn}{Input}
\SetKwInOut{KwOut}{Output}
\SetKwComment{Comment}{$\triangleright$\ }{}

\KwIn{Initial mean estimate $\vec{\mu_0}$, 
Initial covariance estimate $\vec{\Sigma_0}$, 
target distribution $p(\vec{\theta}|\vx)$, number of iterations N, batch size B, learning rate $\epsilon$}

\KwOut{Multivariate normal variational distribution $q_{\bw}(\vec{\theta}):=\mathcal{N}(\vec{\mu}, \vec{\Sigma}$)}

$q_{\bw} \leftarrow \mathcal{N}(\vmu_0, \vSigma_0) $ \;

\For{$i =0 , \ldots,  N-1$ \quad \quad \Comment{iteration loop}  }  { 
    $\{\vtheta^{(0)},\vtheta^{(1)},...,\vtheta^{(B)}\} \sim q_{\vec{w}}(\theta)$ \hfill \Comment*[f]{Sample a batch of B points}\;
    $\mathrm{ELBO} = \sum_j \,
    \log(p(\vtheta^{(j)}, \vx)- \log q_{\vec{w}}(\theta^{(j)})$ \;
   $\vec{w} \leftarrow \vec{w} - \epsilon \nabla_{\vec{w}}\mathrm{ELBO}$ \hfill \Comment*[f]{Optimization step, we use ADAM}\;
}
\caption{Black-box variational inference}
\label{alg:bbvi}
\end{algorithm}


We evaluate the performance of GSM-VI in different settings. 
GSM-VI uses a multivariate Gaussian distribution as its variational family. We separately investigate when the target posterior is in this family and when it is not.

\textbf{Algorithmic details and comparisons.} We compare GSM-VI with a reparameterization variant of BBVI as the baseline, similar to ~\cite{Kucukelbir2016}. BBVI uses the same multivariate Gaussian variational family, which we fit by maximizing the ELBO. (Maximizing the ELBO is equivalent to minimizing KL). We use the ADAM optimizer \cite{Kingma2015} with default settings but vary the learning rate between $10^{-1}$ and $10^{-3}$. We report results only for the best performing setting. The full BBVI algorithm is in Algorithm \ref{alg:bbvi}.

The only free parameter in GSM-VI is the batch size $B$. We find that $B=2$ is better than $B=1$, but there is no improvement beyond that. In all studies, we report results for $B=2$. 

Both algorithms require an initial variational distribution. Unless specified otherwise, we initialize the variational distribution with zero mean and identity covariance matrix.

\textbf{Evaluation metric.} GSM-VI does not explicitly minimize any loss function.
Hence to compare its performance against BBVI, we estimate empirical divergences between the variational and the target distribution and show their evolution with the number of gradient evaluations. In the experiments with synthetic models in Sections~\ref{sec:exp1}, \ref{sec:exp2}, and \ref{sec:exp3} we have access to the true distribution; so we measure the forward KL divergence (FKL) empirically $\big(\mathrm{FKL} = \sum_{{\vtheta_i} \sim p(\vtheta)} \log p(\vtheta_i) - \log q (\vtheta_i)\big)$. To reduce stochasticity, we always use the same pre-generated set of 1000 samples from the target distribution. 
In Section~\ref{sec:exp4}, we do not have access to the samples from the target distribution; so we monitor the negative ELBO. In all experiments, we show the results for 10 independent runs. 

\label{sec:exp1}
\begin{figure}
    \centering
    \includegraphics[width=\textwidth ]{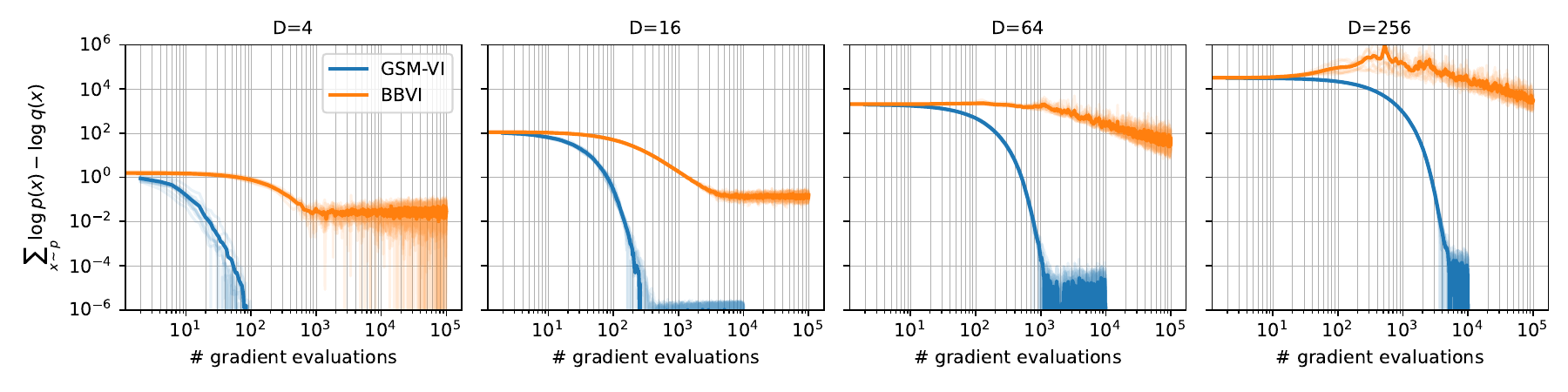}
    \caption{Scaling with dimensions: evolution of FKL with the number of gradient evaluations of the target distribution, which is a  Gaussian distribution with dense covariance. Different panels show the results for different dimensions $D$ of the distribution, specified in the title.
    Translucent lines show the scatter of 10 different runs and the solid line shows the average.}    \label{fig:dims}
\end{figure}

\subsection{GSM-VI for Gaussian approximation}

We begin by studying GSM-VI where the target distribution is also a multivariate Gaussian.

\textbf{Scaling with dimensions.}  How does GSM-VI scale with respect to the dimensions of the sample space? Figure \ref{fig:dims} shows the convergence of FKL for GSM-VI and BBVI as the dimension (D) of the sample space increases. Empirically, we find that the number of iterations required for convergence increases almost linearly with dimensions for GSM. The scaling for BBVI is worse, and it requires 100 times more iterations even for small problems ($D<64$), while also converging to a sub-optimal solution as measured by the FKL metric. 

\label{sec:exp2}
\begin{figure}
    \centering
    \includegraphics[width=\textwidth ]{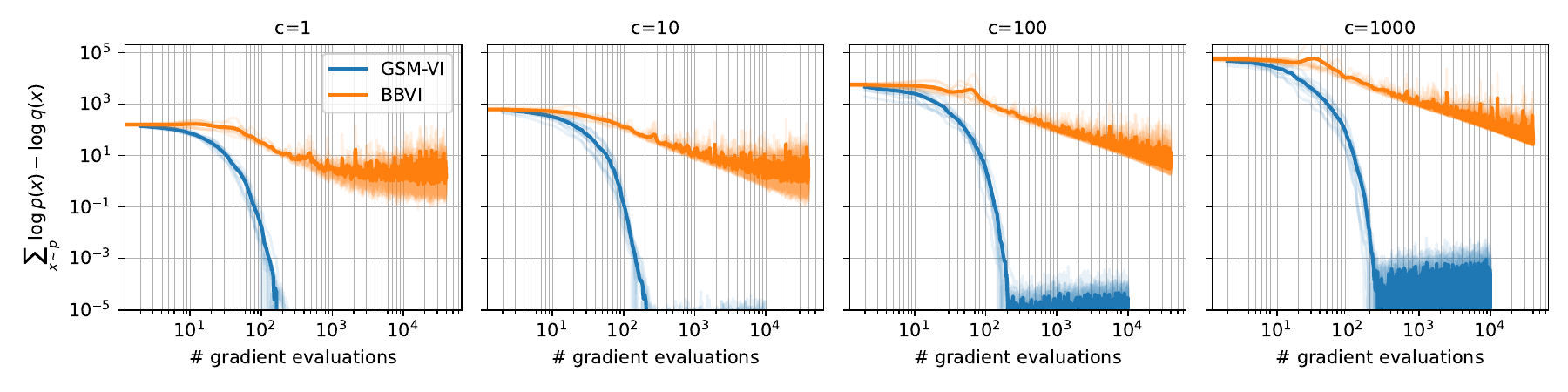}
    \caption{Impact of condition number: evolution of FKL with the number of gradient evaluations of the target distribution. The target is a 10-dimensional Gaussian albeit with a dense covariance matrix of different condition numbers $c$ specified in  the title of different panels. Translucent lines show the scatter of 10 different runs and the solid line shows the average.}
    \label{fig:cond}
\end{figure}

\textbf{Impact of condition number.} What is the impact of the conditioning of the target distribution?  We again consider a Gaussian target distribution, but vary the condition number of its covariance matrix by fixing its smallest eigenvalue to $0.1$, and scaling the largest eigenvalue to $0.1\times c$. Figure \ref{fig:cond} shows the results for a 10 dimensional Gaussian where we vary the condition number $c$ from 1 to 1000. Convergence of GSM-VI seems to be largely insensitive to the condition number of the covariance matrix. BBVI on the other hand struggles with poorly conditioned problems, and it does not converge for $c>100$ even with 100 times more iterations than GSM. 

\label{sec:exp3}
\begin{figure}
    \centering
    \includegraphics[width=\textwidth ]{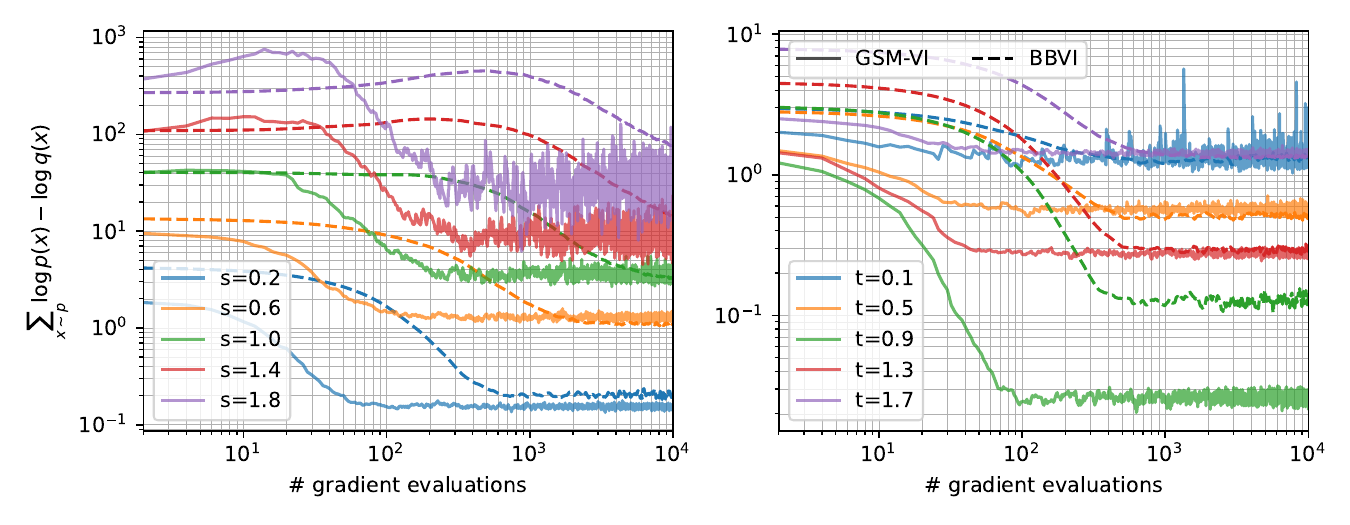}
    \caption{Impact of non-Gaussianity: evolution of FKL with the number of gradient evaluations for Sinh-arcsinh distributions with 10-dimensional dense Gaussian as the base distribution. Gaussian distribution has $s=0,\, t=1$. 
    In the left panel, we vary skewness $s$ while fixing $t=1$, and in the right panel we vary the tail-weight $t$ with skewness fixed to $s=0$.
    Solid lines are the results for GSM, dashed for BBVI.
    }
    \label{fig:nongauss}
\end{figure}

\subsection{GSM-VI for non-Gaussian target distributions}

GSM-VI was designed to solve the exact score-matching equations \Cref{eq:scorematch}, which only have a solution when  the family of variational distributions contains the target distribution (see Lemma~\ref{lem:scorematch}). Here we investigate the sensitivity of GSM-VI to this assumption by fitting non-Gaussian target distributions with varying degrees on non-Gaussianity. Specifically, we suppose that the target has a multivariate \texttt{Sinh-arcsinh} normal distribution \citep{Jones2019}
\begin{equation}
    \vec{z} \sim \mathcal{N}(\vmu, \vSigma); \quad \vec{x} = \mathrm{sinh}\left( \frac{1}{t}\left[\mathrm{sinh}^{-1}(\vec{z}) + s\right]\right)
\end{equation}
where the scalar parameters $s$ and $t$ control, respectively, the skewness and the heaviness of the tails, and the choices $s=0$ and  $t=1$ reduce a Gaussian distribution as a special case.

\Cref{fig:nongauss} shows the result for fitting the variational Gaussian to a 10-dimensional Sinh-arcsinh normal distribution for different values of $s$ and $t$.
As the target departs further from Gaussianity, the quality of variational fit worsens for both GSM-VI (solid lines) and BBVI (dashed lines), but they converge to a fit of similar quality in terms of average FKL. GSM converges to this solution at least 10 times faster than BBVI. For highly non-Gaussian targets ($s \geq 1$ or $|t-1| \geq 0.8$), we have found that GSM-VI does not converge to a fixed point, and it can experience oscillations that are larger in amplitude than BBVI, see for instance $s=1.8$ and  $t=0.1$ on the left and right of Figure \ref{fig:nongauss}, respectively.

\subsection{GSM-VI on real-world data}
\label{sec:exp4}

\begin{figure}
    \centering
    \includegraphics[width=\textwidth ]{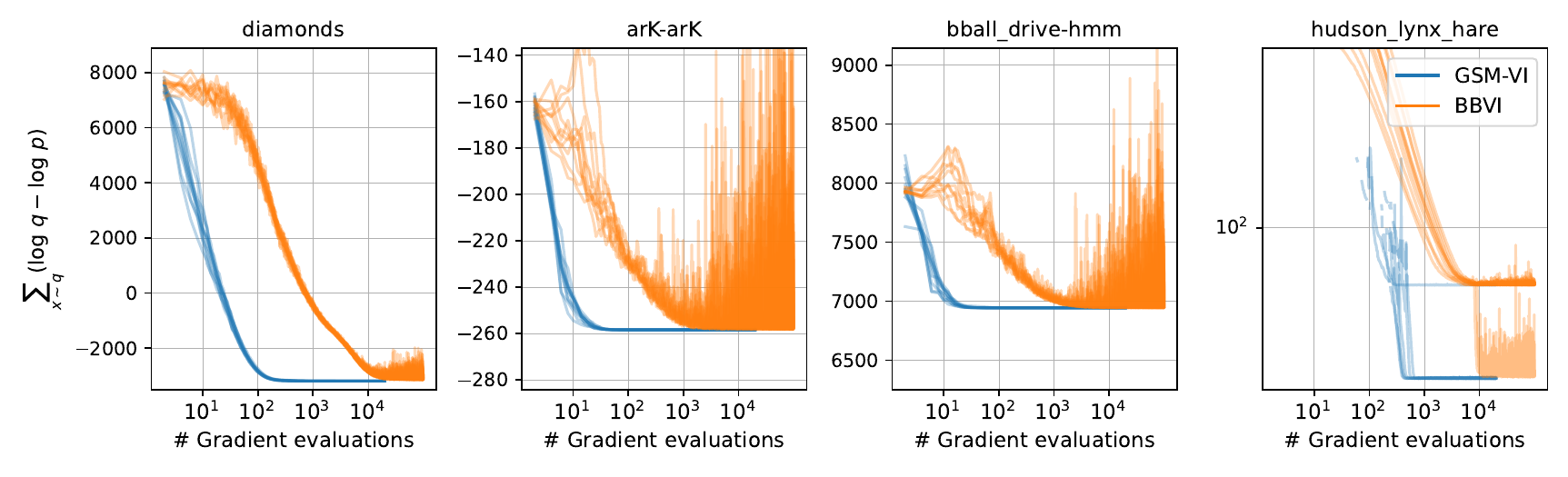}
    \caption{Models from \texttt{posteriordb}: Convergence of the ELBO for four models with multivariate normal posteriors. We show results for 10 runs.
    }
    \label{fig:pdb_g}
\end{figure}

\begin{figure}[t]
    \centering
    \includegraphics[width=\textwidth ]{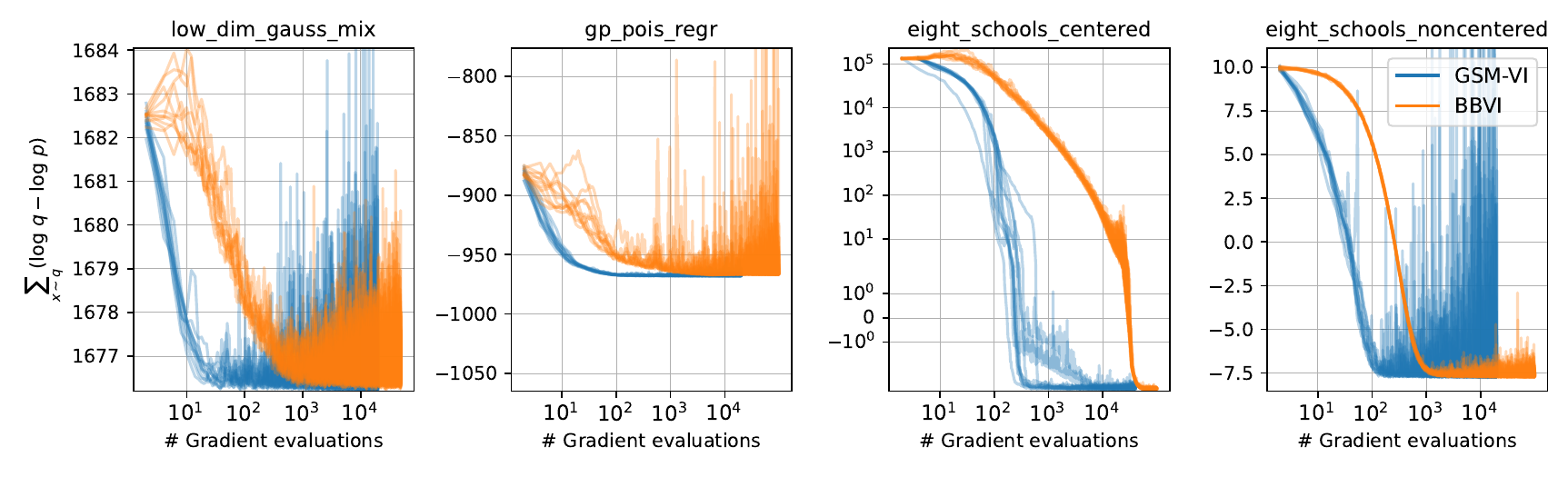}
     \caption{Models from \texttt{posteriordb}: Covergence of ELBO for four models with non-Gaussian posteriors. We show results for 10 runs.
    }   
    \label{fig:pdb_ng}
\end{figure}

We evaluate GSM-VI for approximate on real-world data with 8 models from the \texttt{posteriordb} database \citep{Magnusson2022}. The database provides the \texttt{Stan} code, data and reference posterior samples, and we use \texttt{bridgestan} to access the gradients of these models \citep{Carpenter2012, Roualdes2023}. We study the following models: \texttt{diamonds} (generalized linear models), \texttt{hudson-lynx-hare} (differential equation dynamics), \texttt{bball-drive} (hidden Markov models) and \texttt{arK} (time-series), \texttt{eight-schools-centered} and \texttt{non-centered} (hierarchical meta-analysis),
\texttt{gp-pois-regr} (Gaussian processes), \texttt{low-dim-gauss-mix} (Gaussian mixture).

For each model (except \texttt{hudson-lynx-hare}), we initialize the variational parameter $\vmu_0$ at the mode of the distribution, and we set $\vSigma_0=0.1\,\vec{I}_d$ where $\vec{I}_d$ is the identity matrix of dimension~$d$. 
For \texttt{hudson-lynx-hare}, we initialize the variational distribution as standard normal. 
We also experimented with other initializations. We find that they do not qualitatively change the conclusions, but can have larger variance between different runs. 

We show the evolution of the ELBO for 10 runs of these models.
Four of the models have posteriors that can be fit with multivariate normal distribution:
\texttt{diamonds, hudson-lynx-hare, bball-drive,} and \texttt{arK}. \Cref{fig:pdb_g} shows the result for these models. The other models have non-Gaussian posteriors: \texttt{eight-schools-centered}, \texttt{eight-schools-non-centered}, \texttt{gp-pois-regr,}, and \texttt{low-dim-gauss-mix}. \Cref{fig:pdb_ng} shows the results.

Overall, GSM-VI outperforms BBVI by a factor of 10-100x. When the target posterior is Gaussian, GSM-VI leads to more stable solutions. When the target is non-Gaussian, it converges to the same quality of variational approximation as BBVI. Further, though the ELBO estimate is noisy at the convergence, the 1-D marginals and moments of parameters remain stable.


\section{Conclusion and Future Work}

In this paper we proposed Gaussian score matching VI (GSM-VI), a new approach for VI when the variational family is multivariate Gaussian. GSM-VI is not based on minimizing a divergence or loss function between the variational and target distribution; instead, it repeatedly solves the exact score matching equations with closed-form updates for the mean and covariance matrix of the variational distribution. Our algorithm is implemented in an open-source Python code at \href{https://github.com/modichirag/GSM-VI}{https://github.com/modichirag/GSM-VI}.

Unlike approaches that are rooted in stochastic gradient descent, GSM-VI does not require the tuning of step-size hyper-parameters. It has only one free parameter, the batch size, and we found a batch-size of 2 to perform competitively across all experiments.  Another choice is how to initialize the variational distribution. For the experiments in this paper, we initialized the covariance matrix as the identity matrix, but additional gains could potentially be made with more informed choices derived from a Laplace approximation or L-BFGS Hessian approximation \cite{Zhang2022}. 

We evaluated the performance of GSM-VI on synthetic targets and real-world models from \texttt{posteriordb}. In general, we found that it requires 10-100x fewer gradient evaluations than BBVI for the target distribution to converge. When the target distribution is itself multivariate Gaussian, we observed that GSM-VI scales almost \emph{linearly} with dimensionality, which is significantly better than BBVI, and that GSM-VI is almost insensitive to the condition number of the target covariance matrix. Compared to BBVI, we also found that GSM-VI converges more quickly to a solution with a larger ELBO, which 
is surprising given that BBVI explicitly maximizes the ELBO.


GSM-VI is derived from a principled goal and justification, and the empirical studies indicate that it is a promising method. An important avenue for future work is to provide a proof that GSM-VI converges. We note that good convergence results have been obtained for analogous methods that project onto interpolation equations for empirical risk minimization. For instance the Stochastic Polyak Step achieves the min-max optimal rates of convergence for SGD~\cite{SPS}. Note that convergence of VI is a generally challenging problem, with no known rates of convergence even for BBVI~\cite{Domke2019,Domke20}. 

In another avenue of future work, the score-matching VI idea can potentially be used to design other methods for VI. As one example, we can consider non-Gaussian variational approximations, such as those in the exponential family. As another example, if the variational family is a mixture of Gaussians, we can employ GSM-VI to update the individual components of the mixture. 



\bibliographystyle{plainnat}
\bibliography{references}


\appendix

\section{Proof of Lemma~\ref{lem:scorematch} }
\scorematch*
\begin{proof}
\eqref{eq:interpolation} $\implies $\eqref{eq:scorematch}: Differentiating both sides of \eqref{eq:interpolation} in $\vec{\theta}$ gives
\begin{align*}
   \nabla_{\theta} \log q_{\bw^*}(\vec{\theta}) &=   \nabla_{\vtheta} \log p(\vtheta | \vx) = \nabla_{\theta} \left(\log p(\vtheta, \vx) -\log p(\vx)\right) \\
   & =  \nabla_{\theta} \log p(\vec{\theta, \vx}),\qquad \forall \vtheta \in \Theta.
\end{align*}

\revised{
\eqref{eq:scorematch} $\implies $\eqref{eq:interpolation}: 
Let $\vtheta_0$ be any arbitrary point in $\Theta$. Because $\Theta$ is path-connected, every $\vtheta \in \Theta$ is connected to $\vtheta_0$ via a differentiable path ${\bf r}(t)$ where ${\bf r}(0) =\vtheta_0$ and ${\bf r}(1) = \vtheta.$}
\revised{Integrating both sides of~\eqref{eq:scorematch} along this path ${\bf r}(t)$, using again that $\nabla_{\theta} \log p(\vtheta | \vx) = \nabla_{\theta} \log p(\vtheta , \vx)$, and using the Fundamental Theorem of Calculus gives
\begin{align*}
\log q_{\bw^*}(\vtheta) - \log  q_{\bw^*}(\vtheta_0) &=
    \int_{0}^1   \dotprod{ \nabla_{\theta}\log q_{\bw^*}({\bf r}(t)), {\bf r}'(t) } dt \\
    &= \int_{0}^1   \dotprod{ \nabla_{\theta}\log p(({\bf r}(t)) | \vx), {\bf r}'(t) } dt   \\
    &= \log  p(\vtheta| \vx) - \log  p(\vtheta_0 | \vx), \quad \forall \vtheta \in \Theta.
\end{align*} 
Rearranging and defining $C :=\log  q_{\bw^*}(\vtheta_0 ) - \log  p(\vtheta_0 | \vx) $ gives
\[ \log q_{\bw^*}(\vtheta) = \log p(\vtheta | \vx) + C,\qquad \forall \vtheta \in \Theta, \]
 By exponentiating both sides and integrating in $\vtheta$ over $\Theta$ we have that
\[ 1 = \int_{\vtheta \in\Theta}q_{\bw^*}(\vtheta) d\vtheta= e^{C}\int_{\vtheta \in \Theta} p(\vtheta | \vx)  d\vtheta = e^{C},\]
where we used that $\Theta$ contains the support of both $p(\vtheta | \vx)$ and $q_{\bw^*}(\vtheta)$.}
Consequently $C=0$, which gives our result.
\end{proof}

\section{Proof of Theorem~\ref{theo:GSM} }

Here we give the proof for Theorem~\ref{theo:GSM}. We also re-introduce the theorem with a simplified notation, where we use $(\vec{\mu}_0, \mat{\Sigma}_0)$ to denote the mean and covariance at the previous time step of the method, thus dropping the iteration counter $t$.
\begin{restatable}{theorem}{GSMtheo} \label{theo:GSMapp}
\textbf{(GSM updates)} Let $p(\vec{\vtheta},\vec{x})$ be given for some $\vec{\vtheta}\in\Re^d$, and let $q_{0}(\vec{\vtheta})$ and $q(\vec{\vtheta})$ be the multivariate normal distributions, respectively, with means $\vec{\mu}_0$ and $\vec{\mu}$ and covariance matrices $\mat{\Sigma}_0$ and $\mat{\Sigma}$. 
We seek the distribution
\begin{equation}
\arg\!\min_{\!\!\!\!\!\!\!\vec{\mu},\mat{\Sigma} = \mat{\Sigma}^\top} \bigg[ {\rm KL}(q_0,q)\bigg] \quad\mbox{such that}\quad\nabla_{\vec{\vtheta}} \log q(\vec{\vtheta}) = \nabla_{\vec{\vtheta}} \log p(\vec{\vtheta},\vec{x}).\\[2ex]
\end{equation}
As shorthand, let $\vec{g} \ := \ \nabla_{\vtheta} \log p(\vec{\vtheta},\vec{x}), $
and let $\rho$ be the positive root of the quadratic equation
\begin{equation}
\rho(1\!+\!\rho)\ =\ \vec{g}^\top\mat{\Sigma}_0\vec{g} + \big[(\vec{\mu}_0\!-\!\vec{\vtheta})^\top\!\vec{g}\big]^2.
\end{equation}
Then the solution is given by the following closed-form updates:
\begin{eqnarray}
 \vec{\mu}  & = & \vec{\mu}_0\, +\frac{1}{1+\rho}\left[ \mI - \frac{(\vec{\mu}_0-\vec{\vtheta}) \vec{g}^\top}{1+\rho+ (\vec{\mu}_0-\vec{\vtheta})^\top \vec{g}}\right] \mat{\Sigma}_0\left(\vec{g} -\nabla_{\vec{\vtheta}} \log q_0(\vec{\vtheta_0})\right), \\[2ex]
 \mat{\Sigma} & =& \mat{\Sigma}_0 + (\vec{\mu}_0-\vec{\vtheta})(\vec{\mu}_0-\vec{\vtheta})^\top - (\vec{\mu}-\vec{\vtheta})(\vec{\mu}-\vec{\vtheta})^\top.
\end{eqnarray}
Furthermore,  if $\mat{\Sigma}_0 $ is symmetric positive definite then so is $\mat{\Sigma} $.
\end{restatable}
\begin{proof}
The constraint in this optimization is given by
\begin{eqnarray}
\vec{g} & = & \nabla_{\vec{\theta}} \log q(\vec{\theta}) \\
  & = & \nabla_{\vec{\theta}} \left[-\tfrac{1}{2}(\vec{\theta}-\vec{\mu})\mat{\Sigma}^{-1}(\vec{\theta}-\vec{\mu}) - \tfrac{1}{2}\log\big((2\pi)^d|\mat{\Sigma}|\big)\right] \\
 & = & -\mat{\Sigma}^{-1}(\vec{\theta}-\vec{\mu}).
 \end{eqnarray}
The KL divergence is given by
\begin{equation}
{\rm KL}(q_0,q)\ =\  
  \frac{1}{2}\left\{ \mbox{tr}[\mat{\Sigma}^{-1}\mat{\Sigma}_0] + \log\frac{|\mat{\Sigma}|}{|\mat{\Sigma}_0|}
    + (\vec{\mu}-\vec{\mu}_0)^\top \mat{\Sigma}^{-1}(\vec{\mu}-\vec{\mu}_0)  - d\right\}.
\end{equation}
Dropping irrelevant terms from the optimization, we obtain the Lagrangian
\begin{equation}
{\cal L}(\vec{\mu},\mat{\Sigma},\vec{\lambda})\ =\ \frac{1}{2}\left\{ \mbox{tr}[\mat{\Sigma}^{-1}\mat{\Sigma}_0] - \log|\mat{\Sigma}^{-1}| + (\vec{\mu}-\vec{\mu}_0)^\top \mat{\Sigma}^{-1}(\vec{\mu}-\vec{\mu}_0)\right\}  + \vec{\lambda}^\top\big(\mat{\vec{g}-\Sigma}^{-1}(\vec{\mu}-\vec{\theta})\big).
\end{equation}
It is easier to optimize the matrix $\mat{\Sigma}^{-1}$ instead of $\mat{\Sigma}$.  We can enforce the symmetry of $\mat{\Sigma}^{-1}$ by writing
\begin{equation}
\mat{\Sigma^{-1}}\ =\ \tfrac{1}{2}\left(\mat{\Phi} + \mat{\Phi}^\top\right)
\end{equation}
and performing an unconstrained optimization over $\mat{\Phi}$. With respect to the latter, the gradients of the Lagrangian are given by
\begin{equation}
\frac{\partial {\cal L}}{\partial \Phi_{ij}}\ =\
  \sum_{kl} \left(\frac{\partial {\cal L}}{\partial \Sigma^{-1}_{kl}}\right) \left(\frac{\partial {\Sigma^{-1}_{kl}}}{\partial \Phi_{ij}}\right)\ =\
  \sum_{kl} \left(\frac{\partial {\cal L}}{\partial \Sigma^{-1}_{kl}}\right) \left(\frac{1}{2}\delta_{ki}\delta_{lj} + \frac{1}{2}\delta_{kj}\delta_{li}\right)\ =\ 
  \frac{1}{2}\left(\frac{\partial {\cal L}}{\partial \Sigma^{-1}_{ij}} + \frac{\partial {\cal L}}{\partial \Sigma^{-1}_{ji}}\right).
\end{equation}
Next we examine where the gradients of the Lagrangian vanish:
\begin{eqnarray}
\vec{0} = \frac{\partial{\cal L}}{\partial\vec{\mu}} 
  & \Longrightarrow &  \vec{0}\ =\ \mat{\Sigma}^{-1}(\vec{\mu}-\vec{\mu_0})\, -\, \mat{\Sigma}^{-1}\vec{\lambda}\quad
   \Longrightarrow\quad \boxed{\vec{\lambda} = \vec{\mu} - \vec{\mu}_0}\quad \label{eq:grad_mu} \\ \nonumber  \\
\vec{0} = \frac{\partial{\cal L}}{\partial\vec{\lambda}} 
  & \Longrightarrow & \vec{0}\ =\ \mat{\vec{g}-\Sigma}^{-1}(\vec{\mu}-\vec{\theta})\quad
    \Longrightarrow\quad \boxed{\vec{\mu}-\vec{\theta} = \mat{\Sigma}\vec{g}}\quad  \label{eq:grad_lambda}
    \\ \nonumber \\
\mat{0} = \frac{\partial{\cal L}}{\partial\mat{\Phi}}
  & \Longrightarrow & \mat{0}\ =\ 
   \mat{\Sigma}_0 - \mat{\Sigma} - (\vec{\mu}-\vec{\mu}_0)(\vec{\mu}-\vec{\mu}_0)^\top 
   - \left[ \vec{\lambda}(\vec{\mu}-\vec{\theta})^\top +  (\vec{\mu}-\vec{\theta})\vec{\lambda}^\top\right], \\
   & \Longrightarrow & \boxed{\mat{\Sigma}\ =\ \mat{\Sigma}_0\, +\, (\vec{\mu}-\vec{\mu}_0)(\vec{\mu}-\vec{\mu}_0)^\top\, -\, \vec{\lambda}(\vec{\mu}-\vec{\theta})^\top\, -\, (\vec{\mu}-\vec{\theta})\vec{\lambda}^\top}
   \label{eq:grad_phi}
\end{eqnarray}
We claim that these equations (though nonlinear) can be solved in closed form. The first step is to eliminate $\vec{\lambda}$ from eq.~(\ref{eq:grad_phi}) using eq.~(\ref{eq:grad_mu}). In this way we find
\begin{eqnarray}
\mat{\Sigma} & = & \mat{\Sigma}_0 + (\vec{\mu}-\vec{\mu}_0)(\vec{\mu}-\vec{\mu}_0)^\top 
   - (\vec{\mu}-\vec{\mu}_0)(\vec{\mu}-\vec{\theta})^\top  - (\vec{\mu}-\vec{\theta})(\vec{\mu}-\vec{\mu}_0)^\top \\
   & = & \mat{\Sigma}_0 - \vec{\mu}\vec{\mu}^\top + \vec{\mu}\vec{\theta}^\top + \vec{\theta}\vec{\mu}^\top + \vec{\mu}_0\vec{\mu}_0^\top - \vec{\mu}_0\vec{\theta}^\top - \vec{\theta}\vec{\mu}_0^\top \\
   & = & \mat{\Sigma}_0 + (\vec{\mu}_0-\vec{\theta})(\vec{\mu}_0-\vec{\theta})^\top - (\vec{\mu}-\vec{\theta})(\vec{\mu}-\vec{\theta})^\top. \label{eq:Sigma}
   \end{eqnarray}
It is worth highlighting the form of this equation:
$$\mat{\Sigma}\ =\ \mat{\Sigma}_0 + (\vec{\mu}_0-\vec{\theta})(\vec{\mu}_0-\vec{\theta})^\top - (\vec{\mu}-\vec{\theta})(\vec{\mu}-\vec{\theta})^\top$$
This is a simple rank-two update for $\mat{\Sigma}$.  Note that $\mat{\Sigma}=\mat{\Sigma_0}$ if $\vec{\mu}=\vec{\mu}_0$; also, the solution for $\mat{\Sigma}$ is determined by the solution for $\vec{\mu}$.

Ultimately we must solve for $\vec{\mu}$, but first it is useful to solve for the intermediate quantity $\vec{g}^\top\mat{\Sigma}\vec{g} > 0$. From eq.~(\ref{eq:Sigma}), we obtain
\begin{equation}
\vec{g}^\top\mat{\Sigma}\vec{g}\ =\ \vec{g}^\top\mat{\Sigma}_0\vec{g} + \big[(\vec{\mu}_0-\vec{\theta})^\top\vec{g}\big]^2 - \big[(\vec{\mu}-\vec{\theta})^\top\vec{g}\big]^2,
\end{equation}
and from eq.~(\ref{eq:grad_lambda}), we obtain
\begin{equation}
\vec{g}^\top\mat{\Sigma}\vec{g}\ =\ \vec{g}^\top\mat{\Sigma}_0\vec{g} + \big[(\vec{\mu}_0-\vec{\theta})^\top\vec{g}\big]^2 - \left(\vec{g}^\top\mat{\Sigma}\vec{g}\right)^2.
\label{eq:rho}
\end{equation}
As shorthand, let $\rho = \vec{g}^\top\mat{\Sigma}\vec{g}$. Then from eq.~(\ref{eq:rho}) we see that $\rho$ satisfies the quadratic equation
$$\rho(1\!+\!\rho)\ =\ \vec{g}^\top\mat{\Sigma}_0\vec{g} + \big[(\vec{\mu}_0\!-\!\vec{\theta})^\top\!\vec{g}\big]^2.$$
Note that there are no unknowns on the right side of this equation. The correct solution is given by the positive root since $\rho=\vec{g}^\top\mat{\Sigma}\vec{g}>0$.  Also note that $\rho = (\vec{\mu}-\vec{\theta})^\top\vec{g}$ from eq.~(\ref{eq:grad_lambda}).

It is useful to define one final intermediate quantity before solving for $\vec{\mu}$. Let
$$\vec{\varepsilon}_0\ =\ \mat{\Sigma}_0\vec{g} - \vec{\mu}_0 + \vec{\theta}.$$

Note that $\vec{\varepsilon}_0$ simply measures the degree to which the parameters of $q_0(\vec{\theta})$ violate the desired constraint $\nabla_{\bw} \log q(\vec{\theta}) = \nabla_{\bw}\log p(\vec{\theta},\vec{y})$. Put another way, if $\vec{\varepsilon}_0 = \vec{0}$, then we have the trivial solution $\vec{\mu} = \vec{\mu}_0$ and $\mat{\Sigma}=\mat{\Sigma}_0$. 

Now we have everything to express the solution for $\vec{\mu}$ in a highly intuitive form; in particular, it will be immediately evident that $\vec{\mu} \rightarrow \vec{\mu}_0$ as $\vec{\varepsilon}_0\rightarrow \vec{0}$. Starting from eqs.~(\ref{eq:grad_lambda}) and (\ref{eq:Sigma}), we find
\begin{eqnarray}
\vec{\mu}-\vec{\mu}_0
  & = & \vec{\theta} - \vec{\mu}_0 + \mat{\Sigma}\vec{g}, \\
  & = & \red{\vec{\theta} - \vec{\mu}_0} + \left[\red{\mat{\Sigma}_0} + (\vec{\mu}_0-\vec{\theta})(\vec{\mu}_0-\vec{\theta})^\top - (\vec{\mu}-\vec{\theta})(\vec{\mu}-\vec{\theta})^\top\right]\red{\vec{g}}, \label{eq:nonlinear_mu} \\
    & = & \red{\vec{\varepsilon}_0}\, +\, (\vec{\mu}_0-\vec{\theta})(\vec{\mu}_0-\vec{\theta})^\top\vec{g} - \blue{(\vec{\mu}-\vec{\theta})}(\vec{\mu}-\vec{\theta})^\top\vec{g}, \\
    & = & \vec{\varepsilon}_0\, +\, (\vec{\mu}_0-\vec{\theta})(\vec{\mu}_0-\vec{\theta})^\top\vec{g}\, -\, 
    \blue{(\vec{\mu}-\vec{\mu}_0+\vec{\mu}_0-\vec{\theta})}\green{(\vec{\mu}-\vec{\theta})^\top\vec{g}}, \\
    & = & \vec{\varepsilon}_0\, -\, \green{\rho}(\vec{\mu}-\vec{\mu}_0)\, +\, (\vec{\mu}_0-\vec{\theta})[(\vec{\mu}_0-\vec{\theta}) - (\vec{\mu}-\vec{\theta})^\top]\vec{g}, \\
    & = & \vec{\varepsilon}_0\, -\, \rho(\vec{\mu}-\vec{\mu}_0)\, +\, (\vec{\mu}_0-\vec{\theta})(\vec{\mu}_0-\vec{\mu})^\top\vec{g}, \\
   & = & \vec{\varepsilon}_0\, -\, (\rho\mathbf{I} + (\vec{\mu}_0-\vec{\theta})\vec{g}^\top)(\vec{\mu}_0-\vec{\mu}). \label{eq:linear_mu}
\end{eqnarray}
Note what has happened here: eq.~(\ref{eq:nonlinear_mu}) is a system of {\it nonlinear} equations for $\vec{\mu}$, but in eq.~(\ref{eq:linear_mu}), all the nonlinearity has been expressed in terms of $\rho$. Since $\rho$ can be determined via eq.~(\ref{eq:rho}), we arrive effectively at a system of {\it linear} equations for $\vec{\mu}$. Collecting terms, we obtain
\begin{equation}
\left[(1+\rho)\mathbf{I} + (\vec{\mu}_0-\vec{\theta})\vec{g}^\top\right](\vec{\mu}-\vec{\mu}_0)\ =\ \vec{\varepsilon}_0.
\end{equation}
We thus arrive at the closed-form update
\begin{equation}
\vec{\mu}\ =\ \vec{\mu}_0\, +\, 
\left[(1+\rho)\mathbf{I} + (\vec{\mu}_0-\vec{\theta})\vec{g}^\top\right]^{-1}\vec{\varepsilon_0}
\label{eq:muupadte}
\end{equation}
It is evident from this update that $\vec{\mu} \rightarrow \vec{\mu}_0$ as $\vec{\varepsilon}_0\rightarrow \vec{0}$. The matrix inverse in this update can also be computed efficiently from the Woodbury matrix identity.

In sum, the joint update for $\vec{\mu}$ and $\mat{\Sigma}$ can be efficiently computed as follows:\\[-5ex]
\begin{enumerate}
\itemsep = 0ex
 \item Set $\vec{g} = \nabla_{\bw} \log p(\vec{\theta},\vec{y})$ and $\vec{\varepsilon}_0\ =\ \mat{\Sigma}_0\vec{g} - \vec{\mu}_0 + \vec{\theta}$.
 \item Solve $\rho(1\!+\!\rho)\ =\ \vec{g}^\top\mat{\Sigma}_0\vec{g} + \big[(\vec{\mu}_0\!-\!\vec{\theta})^\top\!\vec{g}\big]^2$ for $\rho>0$.
 \item Compute $\vec{\mu}\ =\ \vec{\mu}_0\, +\, 
\left[(1+\rho)\mathbf{I} + (\vec{\mu}_0-\vec{\theta})\vec{g}^\top\right]^{-1}\vec{\varepsilon_0}$.
\item Compute $\mat{\Sigma}\ =\ \mat{\Sigma}_0 + (\vec{\mu}_0-\vec{\theta})(\vec{\mu}_0-\vec{\theta})^\top - (\vec{\mu}-\vec{\theta})(\vec{\mu}-\vec{\theta})^\top$.
 \end{enumerate}
 Solving the quadratic in 2. for $\rho$ we have the positive 
\begin{center}
\fcolorbox{black}{shade}{$ \displaystyle
 \rho\;=\; \frac{\sqrt{1 +4(\vec{g}^\top\mat{\Sigma}_0\vec{g} + \big[(\vec{\mu}_0\!-\!\vec{\theta})^\top\!\vec{g}\big]^2)} -1 }{2} $}
 \end{center}
 
Solving the above linear equation for $\mu$ and using the Sherman Morrison formula $$\left[a \mI +\vec{u} \vec{g}^\top\right]^{-1} = \frac{1}{a}\left(\mI - \frac{\vec{u} \vec{g}^\top}{a+ \vec{u}^\top \vec{g}}\right), \quad \mbox{for every } \vec{u}, \vec{g}, a$$ gives

\begin{equation}\label{eq:zoe8ho;s9h4z4}
    \vec{\mu}  \; = \; \vec{\mu}_0\, +\frac{1}{1+\rho}\left[ \mI - \frac{(\vec{\mu}_0-\vec{\theta}) \vec{g}^\top}{1+\rho+ (\vec{\mu}_0-\vec{\theta})^\top \vec{g}}\right] \vec{\varepsilon_0}.
\end{equation}

Using that $ \nabla_{\vec{\vtheta}} \log q_0(\vec{\vtheta_0})=-\mat{\Sigma}_0^{-1}(  \vec{\vtheta}- \vec{\mu}_0)$ we have that 
$$\vec{\varepsilon_0} =\mat{\Sigma}_0\left(\vec{g} - \mat{\Sigma}_0^{-1}(\vec{\mu}_0 + \vec{\theta})\right)  =\mat{\Sigma}_0\left(\vec{g}  -\nabla_{\vec{\vtheta}} \log q_0(\vec{\vtheta_0})\right). $$
Finally substituting out $\vec{\varepsilon_0}$ in~\eqref{eq:zoe8ho;s9h4z4} the result
\begin{center}
\fcolorbox{black}{shade}{
$ \displaystyle
\vec{\mu}  \; = \; \vec{\mu}_0\, +\frac{1}{1+\rho}\left[ \mI - \frac{(\vec{\mu}_0-\vec{\theta}) \vec{g}^\top}{1+\rho+ (\vec{\mu}_0-\vec{\theta})^\top \vec{g}}\right] \mat{\Sigma}_0\left(\vec{g}  -\nabla_{\vec{\vtheta}} \log q_0(\vec{\vtheta_0})\right).$
}
\end{center}

{\bf Proof that $\mat{\Sigma_0} \;\; p.s.d \implies \mat{\Sigma} \; \; p.s.d $.}
It remains to prove that our solution for $\mat{\Sigma}$ is positive-definite, or equivalently, that all of its eigenvalues are positive. We begin by rewriting our results for 
$\mat{\Sigma}$ in eq.~(\ref{eq:Sigma})  and $\rho$ in eq.~(\ref{eq:rho}) in a more convenient form. As shorthand, let 
\begin{equation}
\mathbf{M}_0\ =\ \mat{\Sigma}_0 + (\vec{\mu}_0-\vec{\theta})(\vec{\mu}_0-\vec{\theta})^\top,
\end{equation}
so that $\mathbf{M}_0$ captures the first two terms on the right side of eq.~(\ref{eq:Sigma}). Note that $\mathbf{M}_0$ is positive-definite, a fact that we will exploit repeatedly in what follows.  In addition, recall that $\vec{\mu}-\vec{\theta}=\mat{\Sigma}\vec{g}$ from eq.~(\ref{eq:grad_lambda}). Thus with this notation we can rewrite eqs.~(\ref{eq:Sigma}) and~(\ref{eq:rho}) as
\begin{eqnarray}
\mat{\Sigma} & = & \mathbf{M}_0 - (\mat{\Sigma}\vec{g})(\mat{\Sigma}\vec{g})^\top, \label{eq:Sigma2} \\
\rho(1+\rho) & = & \vec{g}^\top\mathbf{M}_0\vec{g}. \label{eq:rho2}
\end{eqnarray}
Now let $\vec{e}$ be any normalized eigenvector of $\mat{\Sigma}$; we want to show that its corresponding eigenvalue~$\lambda_e$ is positive. From eq.~(\ref{eq:Sigma2}), it follows that
\begin{eqnarray}
\lambda_e
  & =  & \vec{e}^\top\mat{\Sigma}\vec{e}\\ 
  & = &  \vec{e}^\top\left[\mathbf{M}_0 - ({\Sigma}\vec{g})(\mat{\Sigma}\vec{g})^\top\right]\vec{e}\\
  & = &  \vec{e}^\top\mathbf{M}_0\vec{e} - \lambda_e^2(\vec{e}^\top\vec{g})^2.
\label{eq:quadratic}
\end{eqnarray}
Note that if $\vec{e}^\top\vec{g}=0$, then it follows trivially that $\lambda_e = \vec{e}^\top \mathbf{M}_0 \vec{e} > 0$. So we only need to consider the non-trivial case $\vec{e}^\top\vec{g}\neq 0$. To proceed, we note that
\begin{equation}
(\vec{e}^\top\mathbf{M}_0\vec{g})^2\ =\ (\vec{e}^\top\mathbf{M}_0^{\frac{1}{2}}\mathbf{M}_0^{\frac{1}{2}}\vec{g})^2\ \leq\ (\vec{e}^\top\mathbf{M}_0\vec{e})(\vec{g}^\top\mathbf{M}_0\vec{g}),
\label{eq:CSineq}
\end{equation}
where we have used the Cauchy-Schwartz inequality to bound $(\vec{e}^\top\mathbf{M}_o\vec{g)}$ in terms of $(\vec{e}^\top\mathbf{M}_0\vec{e})$, the latter of which appears in eq.~(\ref{eq:quadratic}). Substituting this inequality into eq.~(\ref{eq:quadratic}), we find that 
\begin{equation}
\lambda_e\  \geq\ \frac{(\vec{e}^\top\mathbf{M}_0\vec{g})^2}{\vec{g}^\top\mathbf{M}_0\vec{g}}\, -\, \lambda_e^2 (\vec{e}^\top\vec{g})^2.
\label{eq:lambda-inequality}
\end{equation}
To prove that $\lambda_e>0$ we need one more intermediate result. Focusing on the rightmost term in this equality, we note that
\begin{equation}
\lambda_e(\vec{e}^\top\vec{g})\ =\ \vec{e}^\top \Sigma \vec{g}\ =\ \vec{e}^\top\left[\mathbf{M}_0 - (\mat{\Sigma}\vec{g})(\mat{\Sigma}\vec{g})^\top\right]\vec{g}\ =\ \vec{e}^\top\mathbf{M}_0\vec{g} - \lambda_e(\vec{e}^\top\vec{g})(\vec{g}^\top\mat{\Sigma}\vec{g}),
\end{equation}
and rearranging the terms in this equation, we find
\begin{equation}
\vec{e}^\top\mathbf{M}_0\vec{g}\ =\ \lambda_e(\vec{e}^\top\vec{g})(1 + \vec{g}^\top\mat{\Sigma}\vec{g}).
\label{eq:eMv}
\end{equation}
This intermediate result is useful because it relates the two terms on the right side of eq.~(\ref{eq:lambda-inequality}). In particular, using eq.~(\ref{eq:eMv}) to eliminate the term $\vec{e}^\top\mathbf{M}_0\vec{g}$ in eq.~(\ref{eq:lambda-inequality}), we find:
\begin{eqnarray*}
\lambda_e &\geq & \frac{\left[\lambda_e(\vec{e}^\top\vec{g})(1+\vec{g}^\top\mat{\Sigma}\vec{g})\right]^2}{\vec{g}^\top\mathbf{M}_0\vec{g}}\, -\, \lambda_e^2 (\vec{e}^\top\vec{g})^2 \\[2ex]
   & = & \lambda_e^2(\vec{e}^\top\vec{g})^2\left[
              \frac{(1+\vec{g}^\top\mat{\Sigma}\vec{g})^2}{\vec{g}^\top\mathbf{M}_0\vec{g}}\, -\,1\right] \\[2ex]
   & = & \lambda_e^2(\vec{e}^\top\vec{g})^2\left[
              \frac{(1+\rho)^2}{\rho(1+\rho)}\, -\,1\right] \\[1ex]
    & = & \frac{\lambda_e^2 (\vec{e}^\top\vec{g})^2}{\rho}, \\[1ex]
    & > & 0,
\end{eqnarray*}
where the final inequality follows because the individual terms $\lambda_e^2$, $(\vec{e}^\top\vec{g})^2$, and $\rho$ are all strictly positive; note that $\lambda_e$ cannot be equal to zero because this contradicts the equality in eq.~(\ref{eq:quadratic}). This completes the proof. Perhaps it is useful that this derivation also gives upper bounds on $\lambda_e$, namely
\begin{equation}
\frac{1}{\lambda_e}\ \geq\ \frac{(\vec{e}^\top\vec{g})^2}{\rho}\quad\Longrightarrow\quad \lambda_e\ \leq\ \frac{\rho}{(\vec{e}^\top\vec{g})^2}\ =\ \frac{\vec{g}^\top\mat{\Sigma}\vec{g}}{(\vec{e}^\top\vec{g})^2}.
\end{equation}

\end{proof}


\end{document}